\documentclass[runningheads]{llncs}
\usepackage{cite}
\usepackage{amsmath,amssymb,amsfonts}
\usepackage{algorithmic}
\usepackage{bm}
\usepackage{bbm}
\usepackage{mathtools}
\usepackage{graphicx}
\usepackage{textcomp}
\usepackage{xcolor}

\begin{document}
\title{Why Mixup Improves the Model Performance}
%
%
\author{Masanari Kimura\inst{1}\orcidID{0000-0002-9953-3469}}
\authorrunning{Masanari Kimura}
%
\institute{Ridge-i Inc., Tokyo, Japan \\ \email{mkimura@ridge-i.com}\\ 
}
\maketitle              
\begin{abstract}
Machine learning techniques are used in a wide range of domains.
However, machine learning models often suffer from the problem of over-fitting.
Many data augmentation methods have been proposed to tackle such a problem, and one of them is called mixup.
Mixup is a recently proposed regularization procedure, which linearly interpolates a random pair of training examples.
This regularization method works very well experimentally, but its theoretical guarantee is not adequately discussed.
In this study, we aim to discover why mixup works well from the aspect of the statistical learning theory.

\end{abstract}
\section{Introduction}
Machine learning has achieved remarkable results in recent years.
However, despite such excellent performance, machine learning models often suffer from the problem of over-fitting~\cite{lawrence2000overfitting}.
In recent years, a concept called mixup~\cite{zhang2018mixup} has attracted attention as one of the powerful regularization methods for machine learning models.
The main idea of these regularization methods is to prepare
\begin{equation}
    (\tilde{\bm{x}}_{ij}, \tilde{y}_{ij}) = (\lambda\bm{x}_i+(1-\lambda)\bm{x}_j, \lambda y_i + (1-\lambda) y_j)
\end{equation}
mixed with random pairs $(\bm{x}_i, \bm{x}_j)$ of input vectors and their corresponding labels $(y_i, y_j)$ and use them as training data.
This regularization method is very powerful and has been applied in various fields such as image recognition~\cite{tokozume2018between} or speech recognition~\cite{medennikov2018investigation}.
Despite these strong experimental results, there is not enough discussion about why this method works well.


In this paper, we give theoretical guarantees for regularization by mixup and reveal how regularization changes in each setting.
To summarize our results, mixup regularization leads to the following effects:
\begin{itemize}
    \item For linear classifiers, the effect of regularization is higher when the sample size is small, and the sample standard deviation is large.
    \item For neural networks, the effect of regularization is higher when the number of samples is small, and the training dataset contains outliers.
    \item When the parameter $\lambda$ is close to $0$ or $1$, mixup can reduce the variance of the estimator, but this will be affected by bias.
    \item When the parameter $\lambda$ has near the optimal value, mixup can reduce both the bias and variance of the estimator.
    \item Geometrically, mixup reduces the second-order derivative of the convex function that characterizes the Bregman divergence.
\end{itemize}


\section{Related Works}
\subsection{Mixup Variants}
Mixup is originaly proposed by ~\cite{xu2018mixup}.
The main idea of these regularization methods is to prepare
\begin{equation}
    (\tilde{\bm{x}}_{ij}, \tilde{y}_{ij}) = (\lambda\bm{x}_i+(1-\lambda)\bm{x}_j, \lambda y_i + (1-\lambda) y_j) \nonumber
\end{equation}
mixed with random pairs $(\bm{x}_i, \bm{x}_j)$ of input vectors and their corresponding labels $(y_i, y_j)$ and use them as training data, where $\lambda\sim Beta(\alpha, \alpha)$, for $\alpha\in (0, \infty)$.

Because of its power and ease of implementation, several variants have been studied~\cite{pmlr-v97-verma19a,kimICML20}.
However, most of them are heuristic methods and have insufficient theoretical explanations.



\section{Notations and  Preliminaries}
We consider a binary classification problem in this paper.
However, our analysis can easily be applied to a multi-class case.

Let $\mathcal{X}$ be the input space, $\mathcal{Y} = \{-1, +1\}$ be the output space, and $\mathcal{C}$ be a set of concepts we may wish to learn, called concept class.
We assume that each input vector $\bm{x}\in\mathbb{R}^d$ is of dimension $d$.
We also assume that examples are independently and identically distributed~(i.i.d) according to some fixed but unknown distribution $D$.

We consider a fixed set of possible concepts $H$, called hypothesis set.
We receive a sample $B=(\bm{x}_1,\dots,\bm{x}_n)$ drawn i.i.d. according to $D$ as well as the labels $(c(\bm{x}_1),\dots,c(\bm{x}_n))$, which are based on a specific target concept $c\in{\mathcal{C}}:\mathcal{X}\mapsto\mathcal{Y}$.
Our task is to use the labeled sample $B$ to find a hypothesis $h_B\in{H}$ that has a small generalization error with respect to the concept $c$.
The generalization error $\mathcal{R}(h)$ is defined as follows.
\begin{definition}{(Generalization error)}
\label{def:generalization_error}
Given a hypothesis $h\in{H}$, a target concept $c\in\mathcal{C}$, and unknown distribution $D$, the generalization error of $h$ is defined by
\begin{equation}
    \mathcal{R}(h) = \mathbb{E}_{x\sim{D}}\Big[\mathbbm{1}_{h(\bm{x})\neq{c(\bm{x}})}\Big],
\end{equation}
where $\mathbbm{1}_\omega$ is the indicator function of the event $\omega$.
\end{definition}
The generalization error of a hypothesis $h$ is not directly accessible since both the underlying distribution $D$ and the target concept $c$ are unknown
Then, we have to measure the empirical error of hypothesis $h$ on the observable labeled sample $B$. 
\begin{definition}{(Empirical error)}
\label{def:empirical_error}
Given a hypothesis $h\in{H}$, a target concept $c\in\mathcal{C}$, and a sample $B = (\bm{x}_1,\dots,\bm{x}_n)$, the empirical error of $h$ is defined by
\begin{equation}
    \hat{\mathcal{R}}(h) = \frac{1}{n}\sum^n_{i=1}\mathbbm{1}_{h(\bm{x}_i)\neq{c(\bm{x}_i)}}.
\end{equation}
\end{definition}

In learning problems, we are interested in how much difference there is between empirical and generalization errors.
Therefore, in general, we consider the relative generalization error $\hat{\mathcal{R}}(h)-\mathcal{R}(h)$.
\begin{definition}{(Empirical Rademacher complexity)}
\label{def:empirical_rademacher_complexity}
Given a hypothesis set $H$ and a sample $B=(\bm{x}_1,\dots,\bm{x}_n)$, the empirical Rademacher complexity of $H$ is defined as:
\begin{equation}
    \hat{\mathfrak{R}}_B(H) = \mathbb{E}_{\bm{\sigma}}\Big[\sup_{h\in{H}}\frac{1}{n}\sum^n_{i=1}\sigma_i h(\bm{x}_i)\Big],
\end{equation}
where $\bm{\sigma} = (\sigma_1,\dots,\sigma_n)^T$ with Rademacher variables $\sigma_i\in\{-1,+1\}$ which are independent uniform random variables.
\end{definition}
\begin{definition}{(Rademacher complexity)}
\label{def:rademacher_complexity}
Let $D$ denote the distribution according to which samples are drawn. For any sample size $n\geq 1$, the Rademacher complexity of $H$ is the expectation of the empirical Rademacher complexity over all samples of size $n$ drawn according to $D$:
\begin{equation}
    \mathfrak{R}_n(H) = \mathbb{E}_{B\sim{D^n}}\Big[\hat{\mathfrak{R}}_B(H)\Big].
\end{equation}
\end{definition}
Intuitively, this discribes the richeness of hypothesis class $H$.

\begin{figure}[t]
    \centering
    \includegraphics[scale=0.3, bb=10 1 950 371]{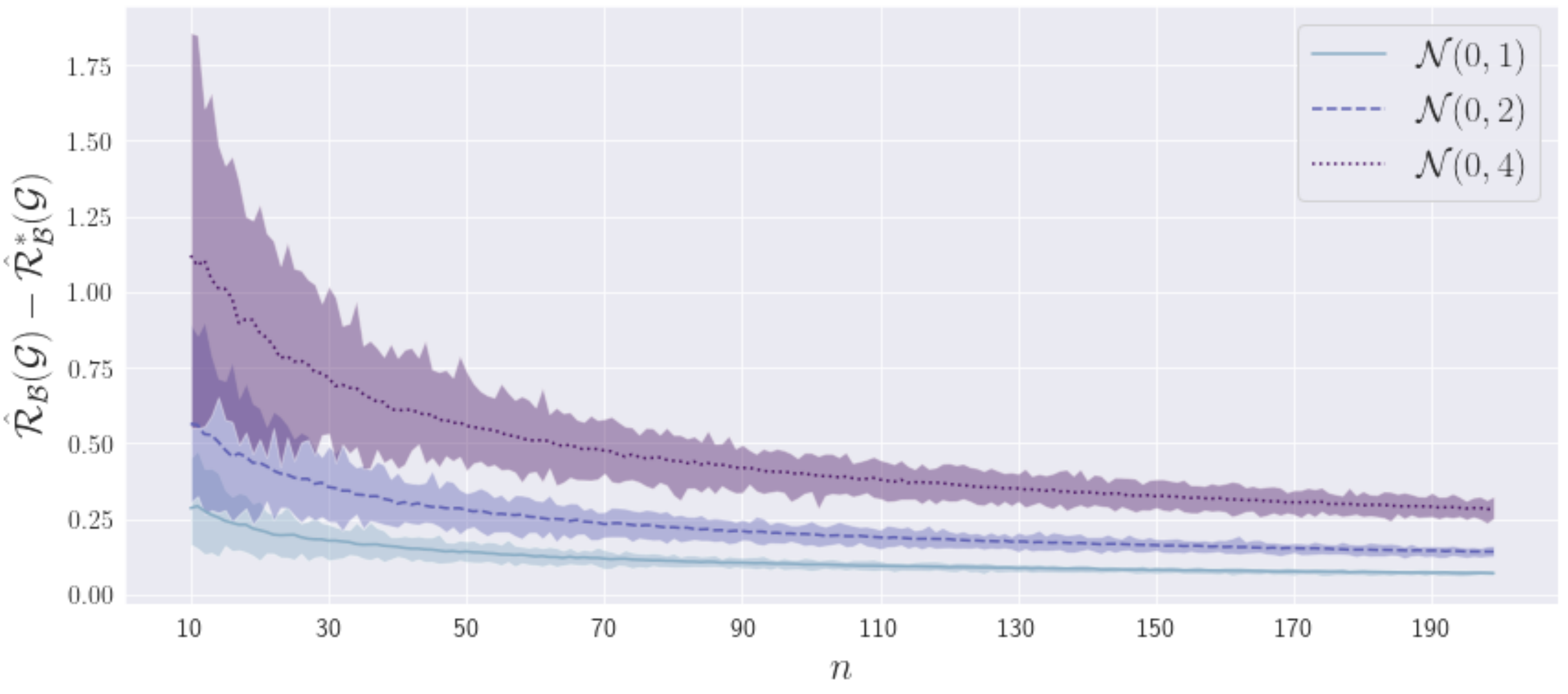}
    \caption{\label{fig:rademacher_diff} The relationship between $\hat{\mathfrak{R}}_B(H_\ell)-\hat{\mathfrak{R}}^*_B(H_\ell)$ and the number of samples $n$ and variance $\sigma^2$ when mixup is applied. Each data point was sampled from the normal distribution $\mathcal{N}(0,\sigma^2)$ and the constant part was set to $1$. 
    }
\end{figure}

The Rademacher complexity is a very useful tool for investigating hypothesis class $H$. 
\begin{lemma}
\label{lm:1}
Let $\mathcal{G}: \mathcal{Z}=\mathcal{X}\times\mathcal{Y} \mapsto [0,1]$ be a family of functions. Then, for any $\delta>0$, with probability at least $1-\delta$, the following holds for all $g\in\mathcal{G}$:
\begin{align}
    \mathbb{E}\Biggl[g(\bm{z})&\leq\frac{1}{n}\sum^n_{i=1}g(\bm{z}_i) + 2\mathfrak{R}_n(G) + \sqrt{\frac{\log{\frac{1}{\delta}}}{2m}}\Biggr] \\
    \mathbb{E}\Biggl[g(\bm{z})&\leq\frac{1}{n}\sum^n_{i=1}g(\bm{z}_i) + 2\mathfrak{R}_B(G) + 3\sqrt{\frac{\log{\frac{2}{\delta}}}{2m}}\Biggr].
\end{align}
\end{lemma}
\begin{proof}
For any sample $B=(\bm{z}_1,\dots,\bm{z}_n)$ and for any $g\in\mathcal{G}$, we denote by $\hat{\mathbb{E}}_B[g]$ the empirical average of $g$ over $B:\hat{\mathbb{E}}_B[g] = \frac{1}{n}\sum^n_{i=1}g(\bm{z}_i)$.
We define the function $\Phi(\cdot)$ for any sample $B$ as follows:
\begin{equation}
    \Phi(B) = \sup_{g\in\mathcal{G}}\mathbb{E}[g] - \hat{\mathbb{E}}_B[g].
\end{equation}
Let $B$ and $B'$ be two samples differing by exactly one point, which mean $\bm{z}_n\in{B}\land\bm{z}_n\notin{B}'$ and $\bm{z}'_n\in{B}'\land\bm{z}'_n\notin{B}$. Then, we have
\begin{align}
    \Phi(B') - \Phi(B) &\leq \sup_{g\in\mathcal{G}}\hat{\mathbb{E}}_B[g] - \hat{\mathbb{E}}_{B'}[g] = \sup_{g\in\mathcal{G}}\frac{g(\bm{z}_n)-g(\bm{z}'_n)}{n}\leq\frac{1}{n} \\
    \Phi(B) - \Phi(B') &\leq \sup_{g\in\mathcal{G}}\hat{\mathbb{E}}_{B'}[g] - \hat{\mathbb{E}}_{B}[g] = \sup_{g\in\mathcal{G}}\frac{g(\bm{z}'_n)-g(\bm{z}_n)}{n}\leq\frac{1}{n}.
\end{align}
Then, by McDiarmid's inequality, for any $\delta>0$, with probability at least $1-\frac{\delta}{2}$, the following holds:
\begin{align}
\Phi(B) &\leq \mathbb{E}_B[\Phi(B)] + \sqrt{\frac{\log{\frac{2}{\delta}}}{2n}} \\
\mathbb{E}_B[\Phi(B)] &\leq \mathbb{E}_{\bm{\sigma},B,B'}\Biggl[\sup_{g\in\mathcal{G}}\frac{1}{n}\sum^n_{i=1}\sigma_i(g(\bm{z}'_i)-g(\bm{z}_i))\Biggr] = 2\mathbb{E}_{\bm{\sigma},B}\Biggl[\sup_{g\in\mathcal{G}}\frac{1}{n}\sum^n_{i=1}\sigma_ig(z_i)\Biggr] 
\nonumber
\end{align}
Then, using MacDiarmid's inequality, with probability $1-\frac{\delta}{2}$,
$
    \mathfrak{R}_n(\mathcal{G})\leq\hat{\mathcal{R}}_B(\mathcal{G}) + \sqrt{\frac{\log{\frac{2}{\delta}}}{2n}}
$.
Finally, we use the union bound and we can have the result of this lemma.
\end{proof}
\begin{lemma}
\label{lm:2}
Let $H$ be a family of functions taking values in $\{-1, +1\}$ and let $\mathcal{G}$ be the family of loss functions associated to $H$: $\mathcal{G} = \{(x, y)\mapsto \mathbbm{1}_{h(x)\neq y}:h\in{H}\}$.
For any samples $B=((\bm{x}_1,y_1),\dots,(\bm{x}_n,y_n))$, let $\mathcal{S_X}$ denote the its projection over $\mathcal{X}:\mathcal{S_X}=(\bm{x}_1,\dots,\bm{x}_n)$.
Then, the following relation holds between the empirical Rademacher complexities of $\mathcal{G}$ and $H$:
\begin{equation}
    \hat{\mathfrak{R}}_B(\mathcal{G}) = \frac{1}{2}\hat{\mathfrak{R}}_\mathcal{S_X}(H).
\end{equation}
\end{lemma}
\begin{proof}
For any sample $B=((\bm{x}_1,y_1),\dots,(\bm{x}_2,y_2))$ of elements in $\mathcal{X}\times\mathcal{Y}$, the empirical Rademacher complexity of $\mathcal{G}$ can be written as:
\begin{align}
\hat{\mathfrak{R}}_B(\mathcal{G}) &= \mathbb{E}_{\bm{\sigma}}\Biggl[\sup_{h\in{H}}\frac{1}{n}\sum^n_{i=1}\sigma_i\mathbbm{1}_{h(\bm{x}_i)\neq y_i}\Biggr] = \frac{1}{2}\mathbb{E}_{\bm{\sigma}}\Biggl[\sup_{h\in{H}}\frac{1}{n}\sum^n_{i=1}\sigma_i h(\bm{x}_i)\Biggr]. 
\end{align}
\end{proof}
\begin{theorem}
\label{thm:rademacher_generalization_bound}
\label{THM:RADEMACHER_GENERALIZATION_BOUND}
Given a hypothesis $h\in{H}$ and the distribution $D$ over the input space $\mathcal{X}$, we assume that $\hat{\mathfrak{R}}_B(H)$ is the empirical Rademacher complexity of the hypothesis class $H$. Then, for any $\delta>0$, with probability at least $1-\delta$ over a sample $B$ of size $n$ drawn according to $D$, each of the following holds over $H$ uniformly:
\begin{eqnarray}
    \mathcal{R}(h) - \hat{\mathcal{R}}(h) &\leq& \hat{\mathfrak{R}}_n(H) + \sqrt{\frac{\log\frac{1}{\delta}}{2m}}, \\
    \mathcal{R}(h) - \hat{\mathcal{R}}(h) &\leq& \hat{\mathfrak{R}}_B(H) + 3\sqrt{\frac{\log\frac{2}{\delta}}{2m}}.
\end{eqnarray}
\end{theorem}
\begin{proof}
From Lemma~\ref{lm:1} and Lemma~\ref{lm:2}, we can have the result of Theorem~\ref{thm:rademacher_generalization_bound} immediately.
\end{proof}



From the above discussion, we can see that if we can quantify the change of empirical Rademacher complexity before and after mixup, we can evaluate the relative generalization error of the hypothesis class $H$.
Our main idea is to clarify the effects of the mixup regularization by examining how these Rademacher complexity changes before and after regularization.
Note that we are not interested in the tightness of the bound, but only in the difference in the bound.


\section{Complexity Reduction of Linear Classifiers with Mixup}
\label{sec:complexity_linear}
In this section, we assume that $H_\ell$ is a class of linear functions:
\begin{equation}
    h(\bm{x})\in{H}_\ell = \Big\{\bm{x}\mapsto\bm{w}^T\bm{x}\ \big|\ \bm{w}\in\mathbb{R}^d,\ \|\bm{w}\|_2\leq\Lambda\Big\},
\end{equation}
where $\bm{w}$ is the weight vector and $\Lambda$ is a constant that regularizes the L2 norm of the weight vector.

\begin{theorem}
\label{thm:rademacher_reduction_linear}
\label{THM:RADEMACHER_REDUCTION_LINEAR}
Given a hypothesis set $H_\ell$ and a sample $B = (\bm{x}_1,\dots,\bm{x}_n)$, we assume that $\hat{\mathfrak{R}}_B(H_\ell)$ is the empirical Rademacher complexity of the hypothesis class $H_\ell$ and $\hat{\mathfrak{R}}^*_B(H_\ell)$ is the empirical Rademacher complexity of $H_\ell$ when mixup is applied.
The difference between the two Rademacher complexity $\hat{\mathfrak{R}}_B(H_\ell)-\hat{\mathfrak{R}}^*_B(H_\ell)$ is less than or equal to a constant multiple of the sample variance of the norm of the input vectors:
\begin{equation}
    \label{eq:rademacher_reduction_linear}
    \hat{\mathfrak{R}}_B(H_\ell)-\hat{\mathfrak{R}}^*_B(H_\ell) \leq \frac{C^\Lambda_\lambda}{\sqrt{n}}\sqrt{s^2\|\bm{x}\|_2},
\end{equation}
where $C^\Lambda_\lambda$ is a constant that depends on the parameter $\lambda$ of mixup and $s^2$ is the sample variance computed from the sample set.
\end{theorem}
\begin{proof}
By the Definition~\ref{def:empirical_rademacher_complexity}, empirical Rademacher complexity of $h(\bm{x}) = \bm{w}^T\bm{x}$ is as follows:
\begin{align}
\hat{\mathfrak{R}}_B(H) &= \mathbb{E}_\sigma\Biggl[\frac{1}{n}\sup_{\|\bm{w}\|_2\leq\Lambda}\sum^n_{i=1}\sigma_i\bm{w}^T\bm{x}_i\Biggr] = \mathbb{E}_\sigma\Biggl[\frac{1}{n}\sup_{\|\bm{w}\|_2\leq\Lambda}\bm{w}^T\sum^n_{i=1}\sigma_i\bm{x}_i\Biggr] \nonumber\\
&= \frac{1}{n}\mathbb{E}_\sigma\Biggl[\sup_{\|\bm{w}\|_2\leq\Lambda}\bm{w}^T\sum^n_{i=1}\sigma_i\bm{x}_i\Biggr] = \frac{1}{n}\mathbb{E}_\sigma\Biggl[\Lambda\Biggl\|\sum^n_{i=1}\sigma_i\bm{x}_i\Biggr\|_2\Biggr] \nonumber\\
&\leq \frac{\Lambda}{n}\Biggl(\mathbb{E}_\sigma\Biggl[\Biggl\|\sum^n_{i=1}\sigma_i\bm{x}_i\Biggr\|_2^2\Biggr]\Biggr)^{\frac{1}{2}} = \frac{\Lambda}{n}\Biggl(\sum^n_{i=1}\|\bm{x}_i\|_2^2\Biggr)^{\frac{1}{2}}. \label{eq:rademacher_1}
\end{align}
Let $\tilde{\bm{x}}_i=\mathbb{E}_{\bm{x}_j}[\lambda\bm{x}_i+ (1-\lambda)\bm{x}_j]$ be the expectation of the linear combination of input vectors by mixup, where $\lambda$ is a parameter in mixup and is responsible for adjusting the weights of the two vectors.
Then, we have
\begin{align}
\hat{\mathfrak{R}}^*_B(H) &\leq \frac{\Lambda}{n}\Biggl(\sum^n_{i=1}\|\tilde{\bm{x}}_i\|_2^2\Biggr)^{\frac{1}{2}} = \frac{\Lambda}{n}\Biggl(\sum^n_{i=1}\Biggl\|\mathbb{E}_{x_j}\Bigl[\lambda\bm{x}_i+(1-\lambda)\bm{x}_j\Bigr]\Biggr\|_2^2\Biggr)^{\frac{1}{2}} \nonumber\\
&= \frac{\Lambda}{n}\Biggl(\sum^n_{i=1}\Biggl\|\lambda\bm{x}_i+(1-\lambda)\mathbb{E}_{x_j}\bigl[\bm{x}_j\bigr]\Biggr\|_2^2\Biggr)^{\frac{1}{2}}  \nonumber\\
&\leq \frac{\Lambda}{n}\Biggl(\sum^n_{i=1}\Bigl(\|\lambda\bm{x}_i\|_2^2 + \Bigr\|(1-\lambda)\mathbb{E}_{\bm{x}_j}[\bm{x}_j]\Bigl\|_2^2\Bigr)\Biggr)^{\frac{1}{2}} \nonumber\\
&= \frac{\Lambda}{n}\Biggl(\lambda^2\sum^n_{i=1}\|\bm{x}_i\|_2^2+(1-\lambda)^2\sum^n_{i=1}\Bigl\|\mathbb{E}_{\bm{x}_j}[\bm{x}_j]\Bigr\|_2^2\Biggr)^{\frac{1}{2}}. \label{eq:rademacher_2}
\end{align}
From \eqref{eq:rademacher_1} and \eqref{eq:rademacher_2}, we can have
\begin{align}
\hat{\mathfrak{R}}_B(H)-\hat{\mathfrak{R}}^*_B(H)
&\leq \frac{\Lambda|1-\lambda|}{n}\Biggl(\sum^n_{i=1}\|\bm{x}_i\|_2^2-\sum^n_{i=1}\Bigl\|\mathbb{E}_{\bm{x}_j}[\bm{x}_j]\Bigr\|_2^2\Biggr)^{\frac{1}{2}} \nonumber \\
&= \frac{\Lambda|1-\lambda|}{\sqrt{n}}\Biggl(\frac{1}{n}\sum^n_{i=1}\|\bm{x}_i\|_2^2-\frac{1}{n}\sum^n_{i=1}\|\bar{\bm{x}}\|_2^2\Biggr)^{\frac{1}{2}} \nonumber \\
&= \frac{\Lambda|1-\lambda|}{\sqrt{n}}\Biggl(s^2(\|\bm{x}\|_2)+\|\bar{\bm{x}}\|_2^2-\|\bar{\bm{x}}\|_2^2\Biggr)^{\frac{1}{2}} \nonumber \\
&= \frac{\Lambda|1-\lambda|}{\sqrt{n}}\sqrt{s^2(\|\bm{x}\|_2)} \geq 0.
\end{align}
\end{proof}

The above results are in line with our intuition and illustrate well how mixup depends on the shape of the data distribution.
As can be seen from the \eqref{eq:rademacher_reduction_linear}, the complexity relaxation by mixup decreases as the number of samples $n$ increases (see Figre~\ref{fig:rademacher_diff}).

\section{Complexity Reduction of neural networks with mixup}
\label{sec:complexity_nn}
Let $H_{L,\bm{W}_L}$ be the function class of a neural network:
\begin{equation}
    h(\bm{x}) \in H_{L,\bm{W}_L} = \Big\{h:\|\bm{v}\|_2 = 1, \prod^L_{i=1}\|\bm{W}_i\|_F\leq \bm{W}_L\Big\},
\end{equation}
where $L$ is the number of layers, $\bm{W}_i$ is the weight matrix, $\bm{v} \in \mathbb{R}^{M_L}$ represents the normalized linear classifier operating on the output of the neural networks with input vector $\bm{x}$ and $\|\bm{A}\|_F$ is the Frobenius norm of the matrix $\bm{A}=(a_{ij})$.

\begin{figure}[t]
    \centering
    \includegraphics[scale=0.28, bb=0 50 950 371]{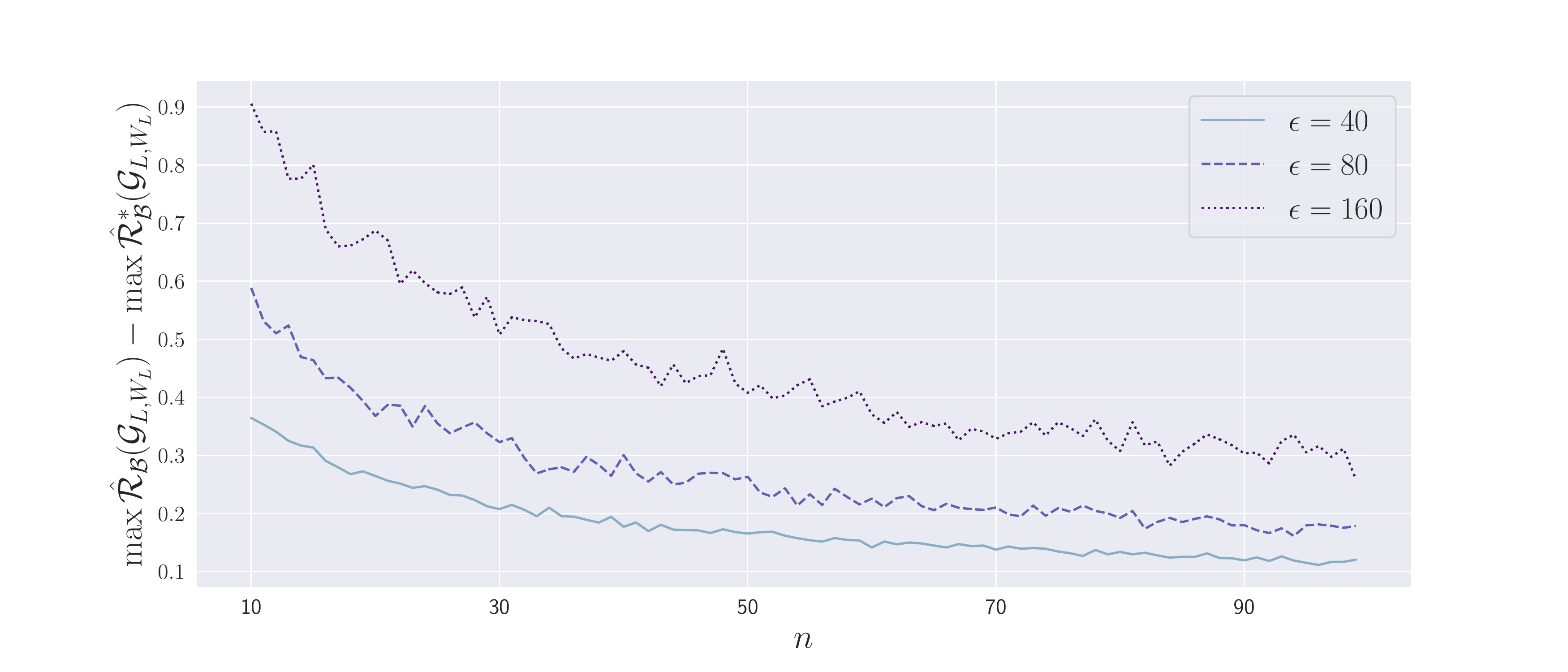}
    \caption{\label{fig:rademacher_diff_nn} The relationship between $\max \hat{\mathfrak{R}}_B(H_{L,\bm{W}_L})-\max\hat{\mathfrak{R}}^*_B(H_{L,\bm{W}_L})$ and the number of samples $n$ and the noise of the outliers $\bm{\epsilon}$.
}
\end{figure}


\begin{theorem}
\label{thm:rademacher_reduction_nn}
\label{THM:RADEMACHER_REDUCTION_NN}
Given a hypothesis set $H_{L,\bm{W}_L}$ and a sample $B = (\bm{x}_1,\dots,\bm{x}_n)$, we assume that $\hat{\mathfrak{R}}_B(H_{L,\bm{W}_L})$ is the empirical Rademacher complexity of the hypothesis class $H_{L,\bm{W}_L}$ and $\hat{\mathfrak{R}}^*_B(H_{L,\bm{W}_L})$ is the empirical Rademacher complexity of $H_{L,\bm{W}_L}$ when mixup is applied. In addition, we assume that each sample $\bm{x}_i$ occurs with the population mean $\bm{\mu_x}$ plus the some noise $\bm{\epsilon}_i$.
In other words, we assume that $\bm{x}_i=\bm{\mu_x}+\bm{\epsilon}_i$.
The difference between the maximum of two Rademacher complexity $\hat{\mathfrak{R}}_B(H_{L,\bm{W}_L})-\hat{\mathfrak{R}}^*_B(H_{L,\bm{W}_L})$ is less than or equal to a constant multiple of the maximum value of noise in a sample of training data when the number of samples $n$ is sufficiently large:
\begin{equation}
    \label{eq:rademacher_reduction_nn}
    \max{\hat{\mathfrak{R}}_B(H_{L,\bm{W}_L})} - \max{\hat{\mathfrak{R}}^*_B(H_{L,\bm{W}_L})} \leq \frac{C^L_\lambda}{\sqrt{n}}\max_i\|\bm{\epsilon}_i\|,
\end{equation}
where $C^L_\lambda$ is a constant that depends on the parameter $\lambda$ of mixup and the number of layers $L$ of neural networks.
\end{theorem}
\begin{proof}
By the upper bound of \cite{neyshabur2015norm}, empirical Rademacher complexity of $h(x)\in{H}_{L,\bm{W}_L}$ is as follows:
\begin{equation}
    \hat{\mathfrak{R}}_B(H_{L,\bm{W}_L}) \leq \frac{1}{\sqrt{n}}2^{L+\frac{1}{2}}\bm{W}_L\max_i\|\bm{x}_i\|.
    \label{eq:rademacher_nn}
\end{equation}
Let $\tilde{\bm{x}}_i=\mathbb{E}_{\bm{x}_j}[\lambda\bm{x}_i+ (1-\lambda)\bm{x}_j]$ be the expectation of the linear combination of input vectors by mixup, where $\lambda$ is a parameter in mixup and is responsible for adjusting the weights of the two vectors.
Then, we have
\begin{align}
\hat{\mathfrak{R}}^*_B(H_{L,\bm{W}_L})&\leq \frac{1}{\sqrt{n}}2^{L+\frac{1}{2}}\bm{W}_L\max_i\|\mathbb{E}_j[\lambda \bm{x}_i + (1-\lambda)\bm{x}_j]\| \nonumber \\
&= \frac{1}{\sqrt{n}}2^{L+\frac{1}{2}}\bm{W}_L\max_i\|\lambda \bm{x}_i + (1-\lambda)\mathbb{E}_j[\bm{x}_j]\| \nonumber\\
&\leq \frac{1}{\sqrt{n}}2^{L+\frac{1}{2}}\bm{W}_L\max_i\Big\{\lambda\|\bm{x}_i\| + (1-\lambda)\|\mathbb{E}_j[\bm{x}_j]\|\Big\}. \nonumber\\ \label{eq:rademacher_nn_mixup}
\end{align}
Now we consider to bound the difference between the maximum values of each quantity, 
\begin{align*}
    \max{\hat{\mathfrak{R}}_B(H_{L,\bm{W}_L})} &= \frac{1}{\sqrt{n}}2^{L+\frac{1}{2}}\bm{W}_L\max_i\|\bm{x}_i\|, \\
    \max{\hat{\mathfrak{R}}^*_B(H_{L,\bm{W}_L})} &= \frac{1}{\sqrt{n}}2^{L+\frac{1}{2}}\bm{W}_L\max_i\Big\{\lambda\|\bm{x}_i\| + (1-\lambda)\|\mathbb{E}_j[\bm{x}_j]\|\Big\},
\end{align*}
and then, from \eqref{eq:rademacher_nn} and \eqref{eq:rademacher_nn_mixup}, and let $\mathcal{J}(H_{L, \bm{W}_L}, B) = \max{\hat{\mathfrak{R}}_B(H_{L,\bm{W}_L})} - \max{\hat{\mathfrak{R}}^*_B(H_{L,\bm{W}_L})} $ we can have
\begin{align}
\mathcal{J}(H_{L, \bm{W}_L}, B) &\leq  \frac{1-\lambda}{\sqrt{n}}2^{L+\frac{1}{2}}\bm{W}_L\max_i\Big|\|\bm{x}_i\|_2 - \|\bar{\bm{x}}\|_2\Big| \nonumber\\
&= \frac{1-\lambda}{\sqrt{n}}2^{L+\frac{1}{2}}\bm{W}_L\max_i\Big|\|\bm{\mu_x}+\bm{\epsilon}_i\|_2 - \|\bar{\bm{x}}\|_2\Big| \nonumber\\
&\leq \frac{1-\lambda}{\sqrt{n}}2^{L+\frac{1}{2}}\bm{W}_L\max_i\Big|\|\bm{\mu_x}\|_2+\|\bm{\epsilon}_i\|_2 - \|\bar{\bm{x}}\|_2\Big| \nonumber\\
&= \frac{1-\lambda}{\sqrt{n}}2^{L+\frac{1}{2}}\bm{W}_L\max_i\|\bm{\epsilon}_i\|_2  \label{eq:40} \geq 0\ \ \ (\because 1-\lambda\geq 0, \|\bm{\epsilon}_i\|_2\geq 0), \nonumber
\end{align}
\end{proof}



According to the above theorem, mixup allows the neural networks robust learning for outliers with accidentally large noise $\bm{\epsilon}$ in the training sample $B$ (see Figre~\ref{fig:rademacher_diff_nn}).

\section{The Optimal Parameters of Mixup}
Here, we let the parameter $\lambda\in(0,1)$.
From \eqref{eq:rademacher_reduction_linear} and \eqref{eq:rademacher_reduction_nn}, we can see that a large $1-\lambda$ has a good regularization effect.
By swapping $i$ and $j$, we can see that $\lambda$ should be close to $0$ or $1$.

In the original mixup paper~\cite{zhang2018mixup}, the parameter $\lambda$ is sampled from the Beta distribution $Beta(\alpha,\alpha)$, where $\alpha$ is another parameter.
We can see that when $\alpha<1$, $\lambda$ is sampled such that one of the input vectors has a high weight (in other words, $\lambda$ is close to $0$ or $1$).
We treated $\lambda$ as a constant in the above discussion, but if we treat it as a random variable $\lambda\sim Beta(\alpha, \alpha)$, we can obtain 
$\mathbb{E}[\lambda] = \frac{\alpha}{\alpha + \alpha} = \frac{1}{2}$ and 
$Var(\lambda) = \frac{\alpha^2}{(\alpha+\alpha)^2(\alpha+\alpha+1)} = \frac{\alpha^2}{4\alpha^2(2\alpha+1)}
= \frac{1}{4(2\alpha+1)},
$
where $\alpha>0$.
Since the $\mathbb{E}[\lambda]$ is a constant, we can see that when the weight parameter $\lambda$ is close to $0$ or $1$, $\alpha$ is expected to be close to $0$.

\begin{figure*}[t]
\centering
\includegraphics[scale=0.3, bb=170 0 963 400]{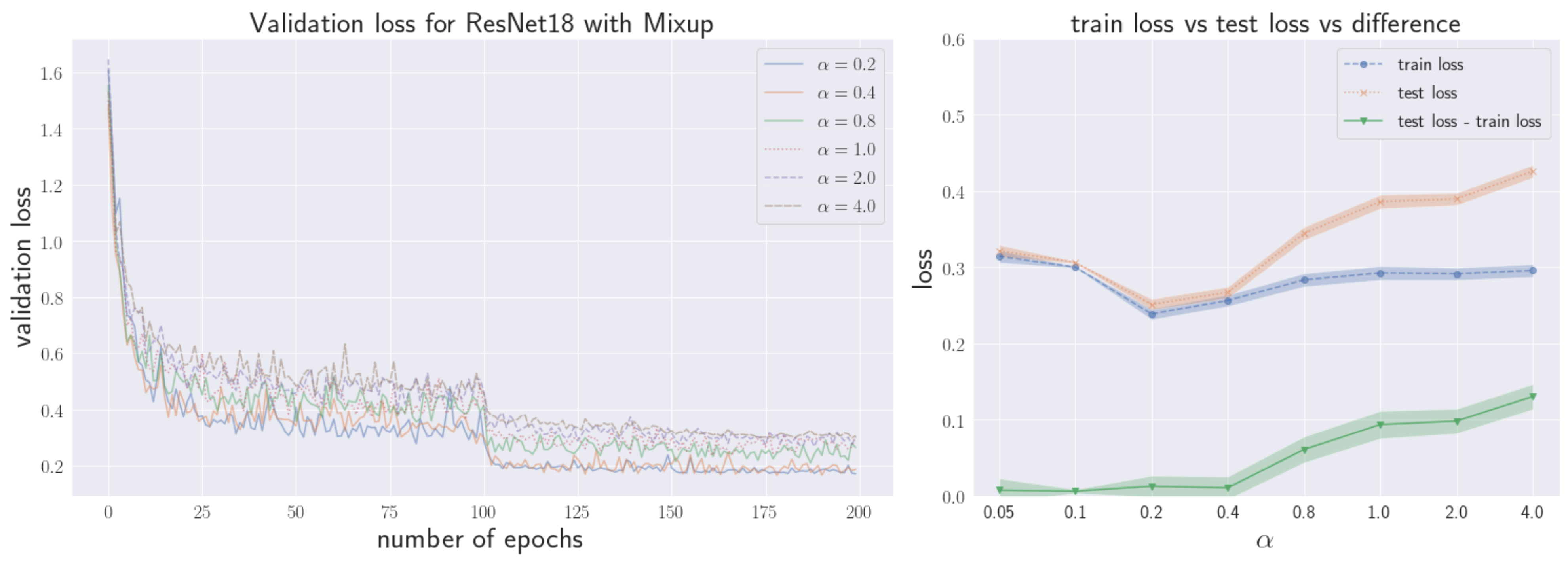}
\caption{\label{fig:cifar_exp} Experimental results for CIFAR-10 dataset.We use ResNet-18 as a classifier and apply mixup with each parameter $\alpha$ for $\lambda\sim Beta(\alpha, \alpha)$. Left: Learning curve of ResNet-18 with mixup. 
}
\end{figure*}

Figure~\ref{fig:cifar_exp} shows the experimental results for CIFAR-10~\cite{krizhevsky2009learning}.
We use ResNet-18~\cite{he2016deep} as a classifier with $lr=0.1$, $epochs=200$ and apply mixup with each parameter $\alpha$ for $\lambda\sim Beta(\alpha, \alpha)$.
In addition, we performed 10 trials with different random seeds and reported the mean values of the trials.
This shows that the generalization performance is higher when the parameter $\alpha$ is a small value.
The right side of Figure~\ref{fig:cifar_exp} shows a plot of the training loss and test loss of the classifier and their differences for each $\alpha$.
We can see that when the value of parameter $\alpha$ is small, the difference between train loss and test loss is small.
Table~\ref{tab:generalization_error} shows the effect of the parameter $\alpha$ on the generalization gap between train and test loss for each dataset.

\begin{table*}[]
\centering
\caption{Effect of the parameter $\alpha$ on the generalization gap between train and test loss for each dataset.}
\label{tab:generalization_error}
\begin{tabular}{l|lllllll}
\hline
dataset       & $\alpha=0.1$    & $\alpha=0.2$ & $\alpha=0.4$ & $\alpha=0.8$ & $\alpha=1.0$ & $\alpha=2.0$ & $\alpha=4.0$ \\ \hline\hline
CIFAR10~\cite{krizhevsky2009learning}       & \textbf{0.006} & 0.012      & 0.010      & 0.061      & 0.093      & 0.098      & 0.130      \\
CIFAR100~\cite{krizhevsky2009learning}      & \textbf{0.182} & 0.259      & 0.277      & 0.292      & 0.348      & 0.596      & 0.695      \\
STL10~\cite{coates2011analysis}          & \textbf{0.013} & 0.0215      & 0.029      & 0.090      & 0.121      & 0.120      & 0.169      \\
SVHN~\cite{netzer2011reading}          & \textbf{0.049} & 0.050      & 0.057      & 0.062      & 0.087      & 0.133      & 0.182     \\ \hline
\end{tabular}
\end{table*}

\section{Geometric Perspective of Mixup Training: Parameter Space Smoothing}
\begin{definition}{(Bregman divergence)}
For some convex function $\varphi(\cdot)$ and $d$-dimensional parameter vector $\bm{\xi}\in\mathbb{R}^d$, the Bregman divergence from $\bm{\xi}$ to $\bm{\xi}'$ is defined as follows:
\begin{equation}
    D_\varphi[\xi:\xi'] = \varphi(\xi) - \varphi(\xi') - \nabla\varphi(\xi')\cdot(\xi-\xi').
\end{equation}
\end{definition}
\begin{theorem}
\label{thm:bregman_gradient_reduction}
\label{THM:BREGMAN_GRADIENT_REDUCTION}
Let $p(\bm{x};\bm{\theta})$ be the exponential distribution family that depends on the unknown parameter vector $\bm{\theta}$.
When mixup is applied, the second-order derivative $\nabla\nabla\psi_\lambda(\bm{\theta})$ of $\psi_\lambda(\bm{\theta})$ that characterizes the Bregman divergence between the parameter $\theta$ and $\theta+d\theta$, which is a slight change of the parameter, satisfies the following:
\begin{equation}
    \nabla\nabla\psi_\lambda(\bm{\theta}) = \lambda^2(\nabla\nabla\psi(\bm{\theta})), \label{eq:bregman_grad}
\end{equation}
where $\psi(\bm{\theta})$ is a convex function of the original data distribution and $\lambda\in(0,1)$ is a parameter of the mixup.
\end{theorem}

\begin{figure*}[t]
    \centering
    \includegraphics[scale=0.3]{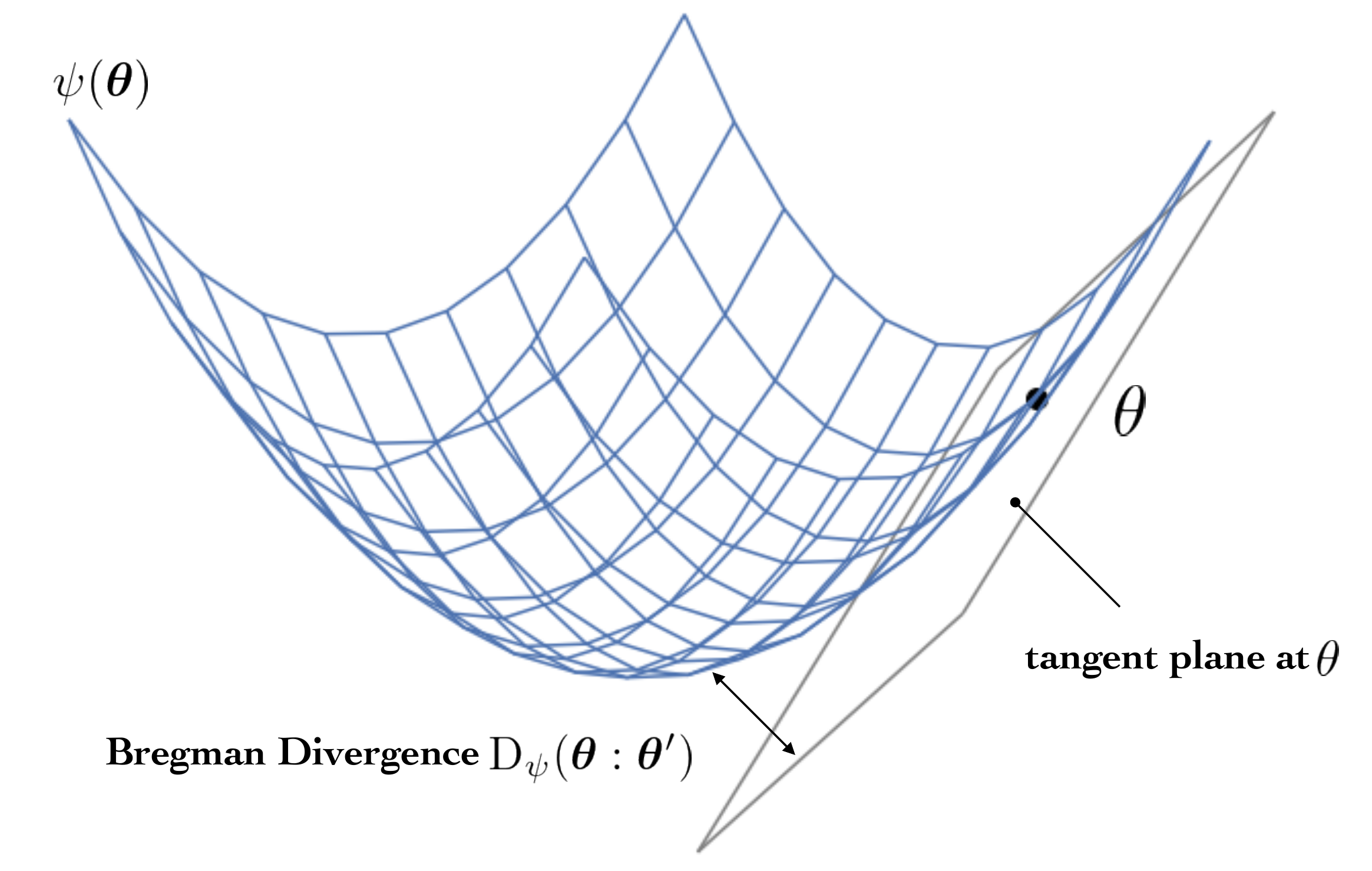}
    \caption{\label{fig:bregman_divergence}Bregman divergence from $\bm{\theta}'$ to $\bm{\theta}$. This divergence derived from the convex function $\psi(\bm{\theta})$ and its supporting hyperplane with normal vector $\nabla\psi(\bm{\theta}_0)$.}
\end{figure*}

\begin{proof}
An exponential family of probability distributions is written as
\begin{equation}
    p(\bm{x};\bm{\theta}) = \exp{\Biggl\{\sum\theta_i x_i + k(\bm{x}) - \psi(\bm{\theta})\Biggr\}},\label{eq:exponential_family_p}
\end{equation}
where $p(\bm{x};\bm{\theta})$ is the probability density function of random variable vector $\bm{x}$ specified by parameter vector $\bm{\theta}$ and $k(\bm{x})$ is a function of $\bm{x}$.
Since $\int p(\bm{x};\bm{\theta})=1$, the normalization term $\psi(\bm{\theta})$ can be written as:
\begin{equation}
    \psi(\bm{\theta}) = \log\int\exp{\Biggl\{\sum_i\theta_i\ x_i + k(\bm{x})\Biggr\}}d\bm{x} \label{eq:exponential_convex_p}
\end{equation}
which is known as the cumulant generating function in statistics.
By differentiating \eqref{eq:exponential_convex_p}, we can confirm that the Hessian becomes a positive definite matrix, which means that $\psi(\bm{\theta})$ is a convex function.
Here, the Bregman divergence from $\bm{\xi}$ to $\bm{\xi}'$ is defined by using the convex function $\varphi(\bm{\xi})$:
\begin{equation}
    D_\varphi[\xi:\xi'] = \varphi(\xi) - \varphi(\xi') - \nabla\varphi(\xi')\cdot(\xi-\xi')
\end{equation}
Let $\psi(\cdot)=\varphi(\cdot)$ and $\bm{\theta}=\bm{\xi}$, then we can naturally define the Bregman divergence for $\psi(\cdot)$ and $\bm{\theta}$.
Differentiating \eqref{eq:exponential_family_p}, we can obtain
\begin{align}
0 &= \frac{\partial}{\partial\theta_i}\int\exp\Biggl\{\sum_i\theta_i x_i + k(\bm{x}) - \psi(\theta) \Biggr\}d\bm{x} \nonumber \\
&= \int\Biggl\{x_i - \frac{\partial}{\partial\theta_i}\psi(\bm{\theta})\Biggr\}p(\bm{x};\bm{\theta})d\bm{x} \label{eq:bregman_3} = \int x_i p(\bm{x};\bm{\theta})d\bm{x} - \frac{\partial}{\partial\theta_i}\psi(\bm{\theta}) \nonumber\\
\therefore \frac{\partial}{\partial\theta_i}\psi(\bm{\theta}) &= \int x_i p(\bm{x};\bm{\theta})d\bm{x} = \mathbb{E}[x_i] \nonumber \\
\nabla\psi(\bm{x}) &= \mathbb{E}[\bm{x}].
\end{align}
Differentiating it again,
\begin{align}
0 &= \int \frac{\partial}{\partial\theta_j}\Big\{x_i - \frac{\partial}{\partial\theta_i}\psi(\bm{\theta})\Big\} p(\bm{x};\bm{\theta}) + \Big\{x_i - \frac{\partial}{\partial\theta_i}\psi(\bm{\theta})\Big\}\frac{\partial}{\partial\theta_j}p(\bm{x};\bm{\theta})d\bm{x} \nonumber \\
&= \int -\frac{\partial^2}{\partial\theta_i\partial\theta_j}\psi(\bm{\theta})d\bm{x}  + \int\Big\{x_i - \frac{\partial}{\partial\theta_i}\psi(\bm{\theta})\Big\}\Big\{x_j - \frac{\partial}{\partial\theta_j}\psi(\bm{\theta})\Big\} p(\bm{x};\bm{\theta})d\bm{x} \nonumber \\
&= -\frac{\partial^2}{\partial\theta_i\partial\theta_j}\psi(\bm{\theta}) + \int(x_i-\mathbb{E}[x_i])(x_j-\mathbb{E}[x_j])p(\bm{x};\bm{\theta})d\bm{x} \nonumber\\
&= -\frac{\partial^2}{\partial\theta_i\partial\theta_j}\psi(\bm{\theta}) + \mathbb{E}[(x_i-\mathbb{E}[x_i])(x_j-\mathbb{E}[x_j])] \nonumber \\
\therefore \nabla\nabla\psi(\bm{\theta}) &= Var(\bm{x}).
\end{align}
Here, if we adopt the linear combination $\tilde{\bm{x}} = \lambda\bm{x} + (1-\lambda)\bm{x}_j$ to find the parameter $\bm{\theta}$, we can obtain
\begin{align}
\nabla\psi_\lambda(\bm{\theta}) &= \mathbb{E}[\tilde{\bm{x}}] = \mathbb{E}[\lambda\bm{x}+(1-\lambda)\mathbb{E}[\bm{x}]] = \mathbb{E}[\bm{x}], \\
\nabla\nabla\psi_\lambda(\bm{\theta}) &= Var(\lambda\bm{x}+(1-\lambda)\mathbb{E}[\bm{x}]) \nonumber \\
&= \lambda^2 Var(\bm{x}) + (1-\lambda)^2 Var(\mathbb{E}[\bm{x}]) = \lambda^2 Var(\bm{x}) = \lambda^2\psi(\bm{\theta})
\end{align}
where $\psi_\lambda(\cdot)$ is defined by
\begin{equation}
    p(\tilde{\bm{x}};\bm{\theta}) = \exp{\Biggl\{\sum\theta_i \tilde{x}_i + k(\tilde{\bm{x}}) - \psi_\lambda(\bm{\theta})\Biggr\}}.
\end{equation}
From Bayes theorem, we would be computing the probability
of a parameter given the likelihood of some data:
    $p(\tilde{\bm{x}};\bm{\theta}) = \frac{p(\tilde{\bm{x}};\bm{\theta})p(\bm{\theta})}{\sum_{\bm{\theta}}'p(\tilde{\bm{x}};\bm{\theta}')p(\bm{\theta}')}$,
and applying mixup means $p(\bm{x};\bm{\theta})\to p(\tilde{\bm{x}};\bm{\theta})$.
And then, we can obtain \eqref{eq:bregman_grad}.
\end{proof}
Bregman divergence is a generalization of KL-divergence, which is frequently used in probability distribution spaces.
The above theorem means that the magnitude of the gradient of the convex function characterizing the Bregman divergence can be smoothed by using the mixup.

\section{Conclusion and Discussion}
In this paper, we provided a theoretical analysis of mixup regularization for linear classifiers and neural networks with ReLU activation functions.
Our results show that a theoretical clarification of the effect of the mixup training.




%
%
%
\bibliographystyle{splncs04}
\bibliography{references}

\begin{thebibliography}{10}
\providecommand{\url}[1]{\texttt{#1}}
\providecommand{\urlprefix}{URL }
\providecommand{\doi}[1]{https://doi.org/#1}

\bibitem{coates2011analysis}
Coates, A., Ng, A., Lee, H.: An analysis of single-layer networks in
  unsupervised feature learning. In: Proceedings of the fourteenth
  international conference on artificial intelligence and statistics. pp.
  215--223 (2011)

\bibitem{he2016deep}
He, K., Zhang, X., Ren, S., Sun, J.: Deep residual learning for image
  recognition. In: Proceedings of the IEEE conference on computer vision and
  pattern recognition. pp. 770--778 (2016)

\bibitem{kimICML20}
Kim, J.H., Choo, W., Song, H.O.: Puzzle mix: Exploiting saliency and local
  statistics for optimal mixup. In: International Conference on Machine
  Learning (ICML) (2020)

\bibitem{krizhevsky2009learning}
Krizhevsky, A., Hinton, G., et~al.: Learning multiple layers of features from
  tiny images  (2009)

\bibitem{lawrence2000overfitting}
Lawrence, S., Giles, C.L.: Overfitting and neural networks: conjugate gradient
  and backpropagation. In: Proceedings of the IEEE-INNS-ENNS International
  Joint Conference on Neural Networks. IJCNN 2000. Neural Computing: New
  Challenges and Perspectives for the New Millennium. vol.~1, pp. 114--119.
  IEEE (2000)

\bibitem{medennikov2018investigation}
Medennikov, I., Khokhlov, Y.Y., Romanenko, A., Popov, D., Tomashenko, N.A.,
  Sorokin, I., Zatvornitskiy, A.: An investigation of mixup training strategies
  for acoustic models in asr. In: Interspeech. pp. 2903--2907 (2018)

\bibitem{netzer2011reading}
Netzer, Y., Wang, T., Coates, A., Bissacco, A., Wu, B., Ng, A.Y.: Reading
  digits in natural images with unsupervised feature learning  (2011)

\bibitem{neyshabur2015norm}
Neyshabur, B., Tomioka, R., Srebro, N.: Norm-based capacity control in neural
  networks. In: Conference on Learning Theory. pp. 1376--1401 (2015)

\bibitem{tokozume2018between}
Tokozume, Y., Ushiku, Y., Harada, T.: Between-class learning for image
  classification. In: Proceedings of the IEEE Conference on Computer Vision and
  Pattern Recognition. pp. 5486--5494 (2018)

\bibitem{pmlr-v97-verma19a}
Verma, V., Lamb, A., Beckham, C., Najafi, A., Mitliagkas, I., Lopez-Paz, D.,
  Bengio, Y.: Manifold mixup: Better representations by interpolating hidden
  states. In: Chaudhuri, K., Salakhutdinov, R. (eds.) Proceedings of the 36th
  International Conference on Machine Learning. Proceedings of Machine Learning
  Research, vol.~97, pp. 6438--6447. PMLR, Long Beach, California, USA (09--15
  Jun 2019), \url{http://proceedings.mlr.press/v97/verma19a.html}

\bibitem{xu2018mixup}
Xu, K., Feng, D., Mi, H., Zhu, B., Wang, D., Zhang, L., Cai, H., Liu, S.:
  Mixup-based acoustic scene classification using multi-channel convolutional
  neural network. In: Pacific Rim Conference on Multimedia. pp. 14--23.
  Springer (2018)

\bibitem{zhang2018mixup}
Zhang, H., Cisse, M., Dauphin, Y.N., Lopez-Paz, D.: mixup: Beyond empirical
  risk minimization. In: International Conference on Learning Representations
  (2018), \url{https://openreview.net/forum?id=r1Ddp1-Rb}

\end{thebibliography}

\end{document}